\documentclass[letterpaper]{article} %
\usepackage{aaai25}  %
\usepackage{times}  %
\usepackage{helvet}  %
\usepackage{courier}  %
\usepackage[hyphens]{url}  %
\usepackage{graphicx} %
\usepackage{natbib}  %
\usepackage{caption} %
\usepackage{algorithm}
\usepackage{algorithmic}

\usepackage{newfloat}
\usepackage{listings}
\DeclareCaptionStyle{ruled}{labelfont=normalfont,labelsep=colon,strut=off} %
\floatstyle{ruled}
\newfloat{listing}{tb}{lst}{}
\floatname{listing}{Listing}
\usepackage{amsmath}
\usepackage{amssymb}
\usepackage{mathtools}
\usepackage{amsthm}
\usepackage{tikz}
\usepackage{bm}
\usepackage{booktabs}
\usepackage[bibliography=common,appendix=append]{apxproof}

\newtheorem{theorem}{Theorem}
\theoremstyle{definition}
\newtheorem{definition}{Definition}[section]
\newtheorem{example}{Example}

\newcommand{\A}{A}

\newcommand{\F}{\mathcal{F}}
\newcommand{\V}{\mathcal{V}}
\newcommand{\Ell}{\mathcal{L}}
\newcommand{\R}{\mathbb{R}}
\newcommand{\N}{\mathbb{N}}

\DeclareMathOperator{\E}{\mathbb{E}}
\newcommand{\AMC}{\text{AMC}}

\title{The Gradient of Algebraic Model Counting}

\author {
    Jaron Maene\textsuperscript{\rm 1},
    Luc De Raedt\textsuperscript{\rm 1,2}
}
\affiliations{
    \textsuperscript{\rm 1}KU Leuven, Belgium\\
    \textsuperscript{\rm 2}\"Orebro University, Sweden\\
    jaron.maene@kuleuven.be, luc.deraedt@kuleuven.be
}

\usepackage{bibentry}

\begin{document}

\maketitle

\begin{abstract}

Algebraic model counting unifies many
inference tasks on logic formulas by exploiting semirings. Rather than focusing on inference, we consider learning, especially in statistical-relational and neurosymbolic AI, which combine logical, probabilistic and neural representations. Concretely, we show that the very same semiring perspective of algebraic model counting also applies to learning.  This allows us to unify various learning algorithms by generalizing gradients and backpropagation to different semirings.
Furthermore, we show how cancellation and ordering properties of a semiring can be exploited for more memory-efficient backpropagation. This allows us to obtain some interesting variations of state-of-the-art gradient-based optimisation methods for probabilistic logical models. We also discuss why algebraic model counting on tractable circuits does not lead to more efficient second-order optimization. Empirically, our algebraic backpropagation exhibits considerable speed-ups as compared to existing approaches.
\end{abstract}
\begin{links}
     \link{Code}{https://github.com/ML-KULeuven/amc-grad}
\end{links}

\section{Introduction}

Algebraic model counting (AMC) generalizes the well-known satisfiability task on propositional logic formulas to semirings \cite{kimmig_algebraic_2017}. Using AMC various probabilistic inference tasks can be solved using the same unifying algorithm, including calculating marginals with evidence,  entropy, and the most probable explanation (MPE). The principles of AMC are reminiscent of those underlying the sum- and max-product algorithms for probabilistic graphical models \cite{friesen_sum-product_2016}.

In the current machine learning age, inference is often combined with learning. This is the focus of statistical-relational learning \cite{de_raedt_statistical_2016}  and neurosymbolic AI \cite{garcez_neurosymbolic_2023}, which aim to integrate the inference and learning paradigms of logic, probability and neural networks.
Our goal is to \emph{extend the unifying algebraic perspective that AMC brought to inference to the learning setting}. To this end, we show how gradients can be generalized over semirings, by taking inspiration from differentiable algebra. This algebraic gradient provides a new toolkit of tractable operations, which can be used to realize a wide variety of learning algorithms, including gradient descent, expectation-maximization, entropy maximization, and low variance gradient estimation.

From a theoretical perspective, the algebraic viewpoint provides a unifying framework and solver for computing many concepts that are used in learning.
This is timely, as \citet{marra_statistical_2024} relate that 
``[t]he situation in neurosymbolic computation today is very much like that of the early days in statistical relational learning, in which there were many competing formalisms, sometimes characterized as the statistical relational learning alphabet soup". \citet{ott_how_2023} even state that ``the field of neurosymbolic AI exhibits a progress-hampering level of fragmentation".  

From a practical perspective, we create a single optimized algorithm that subsumes many existing ones as special cases. For example, the forward-backward algorithm of \citet{darwiche_differential_2003} and the gradient estimator of \citet{de_smet_differentiable_2023} have a worst-case quadratic complexity in the number of nodes. However, these are special cases of our algebraic backpropagation algorithm which has linear complexity. We provide an efficient implementation for the algebraic backpropagation algorithm which takes the semiring properties into account. In our experiments, our implementation outperforms PyTorch and Jax, which are the de facto standard for neurosymbolic learning, by several orders of magnitude.

Finally, we consider algebraic learning algorithms that rely on second-order information. Indeed, empirical evidence of e.g. \citet{liu_scaling_2022} suggests that the existing first-order optimization methods for circuits might be suboptimal. Second-order optimization has quadratic complexity in general, which explains its lack of popularity in machine learning. However, specialized tractable circuit representations such as sd-DNNF have been developed which support many operations in linear time which are otherwise NP-hard \cite{darwiche_knowledge_2002}. Unfortunately, we show that this tractability does not carry over to second-order derivatives and Newton's method is unlikely to be feasible in linear time on these tractable circuits.

In summary, we make the following contributions.
\begin{enumerate}
    \item We introduce the $\nabla$AMC for computing algebraic gradients within a logical semiring framework and show that many tasks related to gradients and learning can be implemented as such an algebraic gradient.
    \item We provide an optimized $\nabla$AMC algorithm for the algebraic gradient by exploiting dynamic programming and semiring properties.
    \item The $\nabla$AMC algorithm is implemented in a user-friendly Rust library called \emph{Kompyle}, and empirically outperforms state-of-the-art algorithms used in neurosymbolic learning
    \item We prove that a second-order algebraic gradient cannot be computed by a circuit in linear time.
\end{enumerate}

\section{Preliminaries}\label{sec:background}

We first review some relevant background on abstract algebra and propositional logic, before turning to algebraic model counting and algebraic circuits.

\subsection{Algebra}

\begin{definition}[Commutative Monoid]
    A commutative monoid $(\A, \odot, e)$ is a set $\A$ with a binary operation $\odot: \A \times \A \to \A$ and an identity element $e \in \A$ such that the following properties hold for all $a$, $b$, and $c$ in $\A$.
    \begin{itemize}
        \item \emph{Associativity}: $(a \odot b) \odot c = a \odot (b \odot c)$
        \item \emph{Commutativity}: $a \odot b = b \odot a$
        \item \emph{Neutrality of Identity}: $e \odot a = a$
    \end{itemize}
\end{definition}

An element $a\in \A$ is \emph{idempotent} when $a \odot a = a$. A monoid is idempotent when all its elements are idempotent.
Note that the identity $e$ is always idempotent ($e \odot e = e$). %

\begin{definition}[Commutative Semiring]
    A commutative semiring $(\A, \oplus, \otimes, e^\oplus, e^\otimes)$ combines two commutative monoids $(\A, \oplus, e^\oplus)$ and $(\A, \otimes, e^\otimes)$ where the following properties hold for all $a$, $b$, and $c$ in $\A$.
    \begin{itemize}
        \item \emph{Distributivity}. $(a \oplus b) \otimes c = (a \otimes c) \oplus (b \otimes c)$
        \item \emph{Absorption}. $e^\oplus \otimes a = e^\oplus$
    \end{itemize}
\end{definition}

We will from on now write monoid or semiring for brevity, and leave the commutativity implied.
A \emph{ring} is a semiring with additive inverses, meaning that for every $a$ in $\A$, there is an inverse element $-a$ such that $a \oplus -a = e^{\oplus}$. 

A simple example of a semiring is the Boolean semiring, which has true ($\top$) and false ($\bot$) as the domain and uses the logical OR and AND operations as sum and product, respectively. Table~\ref{tab:semirings} gives an overview of relevant semirings along with the shorthand name we employ for them. For example, we denote the Boolean semiring as $\textsc{Bool}$.

\begin{table*}
    \centering
    \begin{tabular}{c c c c c c l l l}
        Semiring & Domain $\A$  & $\oplus$ & $\otimes$ & $e^\oplus$ & $e^\otimes$ & Labels $\alpha(x)$ & AMC Task & $\nabla$AMC Task \\\hline
        \textsc{Bool} & $\{\top, \bot\}$ & $\lor$ & $\land$ & $\bot$ & $\top$ & $\top$ & SAT & Cond. SAT \\
         &  &  &  &  &  & $\sim p(x)$ & Sampling & IndeCateR \\

        \textsc{Nat} & $\N$ & $+$ & $\times$ & $0$ & $1$ & 1 & \#SAT & Cond. \#SAT \\
        \textsc{Prob} & $\R_{\geq 0}$ & $+$ & $\times$ & $0$ & $1$ & $p(x)$ & WMC & Gradient \\
        \textsc{Log} & $\{{-}\infty\} {\cup} \R_{\leq 0}$ & logaddexp & $+$ & $-\infty$ & $0$ & $\log p(x)$ & Log WMC & Log Gradient \\
        \textsc{Viterbi} & $\R_{\geq 0}$ & $\max$ & $\times$ & $0$ & $1$ & $p(x)$ & MPE & MGE \\
        \textsc{Tropical} & $\R_{\geq 0}$ & $\max$ & $+$ & $-\infty$ & $0$ & $\log p(x)$ & Log MPE & Log MGE \\
        \textsc{Fuzzy} & $[0, 1]$ & $\max$ & $\min$ & $0$ & $1$ & $p(x)$ & Fuzzy & / \\
        \textsc{Grad}$_y$ & $\R_{\geq 0} \times \R$ & Eq.~\ref{eq:grad_plus} & Eq.~\ref{eq:grad_times} & $(0,0)$ & $(0, 1)$ & $(p(x), \frac{\partial p(x)} {\partial y})$ & Gradient & Hessian \\ 
        \textsc{Sens} & $\R[\V]$ & $+$ & $\times$ & $0$ & $1$ & x & Sensitivity & Cond. Sens. \\ 
        \textsc{Obdd} & OBDD($\V$) & $\lor$ & $\land$ & OBDD(0) & OBDD(1) & OBDD($x$) & OBDD & Cond. OBDD
    \end{tabular}
    \caption{Overview of relevant commutative semirings with their corresponding interpretation in AMC and $\nabla \AMC$.}
    \label{tab:semirings}
\end{table*}

\subsection{Propositional Logic}

\paragraph{Syntax} We write $\V$ for a set of propositional variables. The infinite set of logical formulas $\F_\V$ over the variables $\V$ is defined inductively as follows. A formula $\phi \in \F_\V$ is either true $\top$, false $\bot$, a propositional variable $v \in \V$, a negation of a formula $\neg \phi_1$, a conjunction of formulas $\phi_1 \land \phi_2$, or a disjunction of formulas $\phi_1 \lor \phi_2$. \emph{Literals} are variables or negated variables, and the set of all literals is denoted as $\Ell = \V\ \cup\ \{\neg v \mid v\in\V\}$. 

\begin{definition}
Given a formula $\phi$, the \textit{conditioned} formula $\phi {\mid} x$ with $x \in \Ell$ equals the formula $\phi$ where every occurrence of the literal $x$ is replaced with $\top$ and every occurrence of $\neg x$ is replaced with $\bot$. When $x \not\in \Ell$, $\phi {\mid} x$ is defined as $\bot$.
\end{definition}

\paragraph{Semantics} An \textit{interpretation} $I \subset \Ell$ is a set of literals which denotes a truth assignment. This means that for each variable $v \in \V$, either $v \in I$ or $\neg v \in I$. When a formula $\phi$ is satisfied in the interpretation $I$ according to the usual semantics, we say that $I$ is a \textit{model} of $\phi$, written as $I \models \phi$. The set of all models of a formula is denoted $\mathcal{M}(\phi) = \{ I \mid I \models \phi \}$.

From the algebraic view, the propositional formulas also form a semiring $(\F_\V, \lor, \land, \bot, \top)$. As opposed to the \textsc{Bool} semiring, the operations $\lor$ and $\land$ are structural here and create new formulas from their arguments. This is also known as the \emph{free} commutative semiring generated by the literals $\Ell$.

\subsection{Algebraic Model Counting}

The task of \emph{algebraic model counting} (AMC) consists of evaluating the models of a formula in a given semiring \cite{kimmig_algebraic_2017}.

\begin{definition}[Algebraic Model Counting]\label{def:amc}
    Given a semiring $(\A, \oplus, \otimes, e^\oplus, e^\otimes)$ and a labelling function $\alpha: \Ell \to \A$ which maps literals into the semiring domain, the algebraic model count is a mapping from formulas into the semiring domain.
    \[ \text{AMC}(\phi;\alpha) = \bigoplus_{I \in \mathcal{M}(\phi)} \bigotimes_{x \in I} \alpha(x) \]
\end{definition}

AMC generalizes many existing inference tasks. For example, the satisfiability (SAT) task, which asks whether a formula has a model, can be solved by $\AMC$ in the \textsc{Bool} semiring. Model counting (\#SAT), which asks how many models a formula has, is solved using the \textsc{Nat} semiring, and weighted model counting (WMC) using the \textsc{Prob} semiring. WMC is of particular significance, as probabilistic inference in e.g. Bayesian networks can be reduced to WMC \cite{chavira_probabilistic_2008}.

\begin{example}\label{ex:amc}
Consider the formula $\phi = (x \lor y) \land z$ over the set of variables $\V = \{x, y, z\}$. This formula has three models: $\mathcal{M}(\phi) = \{ \{x, y, z\}, \{\neg x, y, z\}, \{x, \neg y, z\} \}$. So for the AMC, we get
\begin{align*}
\AMC(\phi; \alpha) =& \ \left(\alpha(x) \otimes \alpha(y) \otimes \alpha(z)\right) \\ 
&\oplus \left(\alpha(\neg x) \otimes \alpha(y) \otimes \alpha(z)\right) \\
&\oplus \left(\alpha(x) \otimes \alpha(\neg y) \otimes \alpha(z)\right)
\end{align*}

To compute the model count, we evaluate in the $\textsc{Nat}$ semiring with a constant labelling function $\forall x \in \Ell: \alpha(x) = 1$, giving $\AMC_{\textsc{Nat}}(\phi;\alpha) = 3$. Similarly, setting all weights in $\alpha$ to $\top$ in the \textsc{Bool} semiring shows that $\phi$ is satisfiable.
If we assign weights to the literals, e.g.
\begin{align*}
&\alpha(x) = 0.5,\ \alpha(\neg x) = 0.5,\ \alpha(y)=0.1,\\ 
&\alpha(\neg y)=0.9, \ \alpha(z)=0.8, \ \alpha(\neg z) = 0.2
\end{align*}
we get $\AMC_\textsc{Prob}(\phi; \alpha) = 0.44 $ for the WMC.

\end{example}

The algebra of propositional logic (the \textit{Boolean algebra}, not to be confused with the Boolean semiring) observes additional properties on top of the semiring such as idempotency. This difference between the Boolean algebra and the free semiring is precisely what makes AMC hard; two equivalent formulas might not be equivalent under a specific semiring. For example, $\phi \land \phi$ equals $\phi$ in the Boolean algebra but not in the free semiring.

\subsection{Circuits}

By reusing subformulas, circuits are a more compact representation for Boolean formulas \cite{vollmer_introduction_1999}.

\begin{definition}[Boolean Circuit]
    A Boolean circuit is a directed acyclic graph representing a propositional formula. This means that every leaf node contains a literal, and all other nodes are either $\lor$-nodes or $\land$-nodes.
\end{definition}

\emph{Algebraic circuits} generalize Boolean circuits to semiring operations \cite{derkinderen_algebraic_2020}. Algebraic circuits in the $\textsc{Prob}$ semiring are better known as \emph{arithmetic circuits}.

\begin{example}\label{ex:circuit}
The formula $\phi = (x \lor y) \land z$ can be represented by the algebraic circuit below.
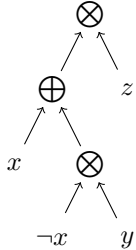
\begin{center}
\begin{tikzpicture}
\tikzstyle{mystyle}=[circle,minimum size=4mm,draw=black,fill=white]
    \node[shape=circle] (x) at (0,0) {$x$};
    \node[shape=circle, inner sep=0pt] (c) at (1,0) {$\bigotimes$};
    \node[shape=circle] (z) at (1.5,1) {$z$};
    \node[shape=circle, inner sep=0pt] (a) at (0.5,1) {$\bigoplus$};
    \node[shape=circle, inner sep=0pt] (b) at (1,2) {$\bigotimes$};
    \node[shape=circle] (nx) at (0.5,-1) {$\neg x$};  
    \node[shape=circle] (y) at (1.5,-1) {$y$};  
    \path [->] (x) edge (a);
    \path [->] (c) edge (a);
    \path [->] (z) edge (b);
    \path [->] (a) edge (b);
    \path [->] (nx) edge (c);
    \path [->] (y) edge (c);
\end{tikzpicture}
\end{center}
\end{example}

The tractability of queries on a circuit are related to the structural properties of that circuit \cite{darwiche_knowledge_2002, vergari_compositional_2021}. For example, \emph{determinism} is required to compute the model count in polynomial time. A circuit is deterministic if all the children of an $\lor$-node are mutually exclusive, meaning they do not share any model. The circuit in Example~\ref{ex:circuit} is deterministic.

Transforming circuits to achieve structural properties such as determinism is achieved with \emph{knowledge compilation} \cite{darwiche_knowledge_2002}. \citet{kimmig_algebraic_2017} proved that the properties of the semiring determine are linked to structural properties of the algebraic circuit. For example, determinism is needed to evaluate in a semiring that is not additively idempotent (of which model counting in the \textsc{Nat} semiring is an example). From an algebraic viewpoint, knowledge compilation tries to generate the smallest circuit representing the models $\mathcal{M}(\phi)$ within a specific semiring.

\section{Conditionals as Gradients}\label{sec:amc_grad}

Due to the inclusion of neural networks, neurosymbolic methods are typically trained using gradient descent. As probabilistic inference is differentiable, these neurosymbolic models are end-to-end trainable with off-the-shelf optimizers such as Adam \cite{kingma_adam_2017}. Some examples of this approach include the semantic loss \cite{xu_semantic_2018} and DeepProbLog \cite{manhaeve_deepproblog_2018}. Other statistical-relational techniques are frequently trained by expectation-maximization (EM), which is also closely linked to gradients \cite{xu_convergence_1996}.
For these reasons, we focus on the computation of gradients. 

It is known that the gradient of probabilistic inference is the same as conditional inference \cite{darwiche_differential_2003}. In AMC with the $\textsc{Prob}$ semiring, better known as weighted model counting, we have that
\[ \frac {\partial \AMC_{\textsc{Prob}}(\phi; \alpha)} {\partial \alpha(x)} = \AMC_{\textsc{Prob}}(\phi {\mid} x; \alpha) \]

We hence propose to generalize the notion of gradients to algebraic model counting as follows. 
\begin{definition} The AMC gradient is defined as the vector of AMC conditionals to each literal $x_i \in \Ell$.
\[ \nabla \AMC(\phi;\alpha) = [ \AMC(\phi {\mid} x_1;\alpha), \dots, \AMC(\phi {\mid} x_n;\alpha)]^\top \]
\end{definition}
Observe that $\nabla \AMC$ is well-defined in any semiring, even when the semiring domain is discrete or otherwise non-differentiable such as in the \textsc{Bool} or \textsc{Nat} semirings.

\paragraph{Literal vs Variable Gradients}
Our definition of $\nabla \AMC$ is the gradient towards a literal and not a variable. However, when maximizing the labels of the positive and negative literal separately, the global optimum is trivial. In practice, the positive and negative labels of a variable $v$ are often linked, e.g. as $\alpha(v) \oplus \alpha(\neg v) = e^\otimes$. In this case, it follows that the AMC derivative to a variable $v$ is
\[ \AMC(\phi {\mid} v; \alpha) \oplus -\AMC(\phi {\mid} \neg v; \alpha) \]

We will not further make this distinction, as the above can be computed straightforwardly from the AMC gradient from Definition~\ref{def:amc} and is only well-defined on rings.

\begin{example}\label{ex:amc_grad}
Consider again $\phi = (x \lor y) \land z$, the formula of Example~\ref{ex:amc}. Then $\nabla \AMC(\phi; \alpha)$ is
\begin{align*}
[ &\AMC(\phi {\mid}x;\alpha), \dots, \AMC(\phi {\mid}\neg z; \alpha) ]^\top \\
= [ &\alpha(y) {\otimes} \alpha(z) \oplus \alpha(\neg y) {\otimes} \alpha(z),\\
&\alpha(x) {\otimes} \alpha(z) \oplus \alpha(\neg x) {\otimes} \alpha(z),\\
&\alpha(x) {\otimes} \alpha(y) \oplus \alpha(\neg x) {\otimes} \alpha(y) \oplus \alpha(x) {\otimes} \alpha(\neg y),\\
&\alpha(y) {\otimes} \alpha(z),\\
&\alpha(x) {\otimes} \alpha(z),\\
&\alpha(e^\oplus)
]
\end{align*}

\end{example}

\subsection{Conditionals as Semiring Derivations}

We further motivate the definition of $\nabla \AMC$ by its relation to differentiable algebra \cite{kolchin_differential_1973}. Differentiable algebra studies generalizations of derivatives, through functions which observe the same linearity and product rule as conventional derivatives.

\begin{definition}[Semiring Derivation]
A derivation $\delta$ on a semiring $(\A, \oplus, \otimes, e^\oplus, e^\otimes)$ is a map $\delta: \A \to \A$ on itself which satisfies the following properties for all $a$ and $b$ in $\A$.
\begin{itemize}
    \item \emph{Linearity}: $\delta (a \oplus b) = \delta(a) \oplus \delta(b) $
    \item \textit{Product rule}: $\delta (a \otimes b) = \left( a \otimes \delta (b)\right) \oplus \left(b \otimes \delta (a)\right)$
\end{itemize}
\end{definition}

Many properties that hold for conventional derivation are retained for semiring derivations. For example, $\delta(e^\oplus) = \delta(e^\otimes) = e^\oplus$ in any derivation. We refer to e.g. \citet{dimitrov_derivations_2017} for a summary of existing results on semiring derivations.
Interestingly, semiring derivations themselves induce an algebraic structure, which gets generated by $\nabla \AMC$.

\begin{theorem}\label{thm:derivation_module}
Every derivation $\delta$ is a linear combination of the elements in {\normalfont $\nabla \AMC$}. More formally, {\normalfont $\nabla \AMC$} is a basis of the $\F_\V$-semimodule over $\mathcal{D}(\F_\V)$.
\end{theorem}
\begin{proof}
    In appendix.
\end{proof}

Here, we denote the set of all possible derivations on $\A$ as $\mathcal{D}(\A)$.
Theorem~\ref{thm:derivation_module} says that every semiring derivation in $\mathcal{D}(\F_\V)$ can be seen as computing a dot product with $\nabla \AMC$. So in this sense, calculating $\nabla \AMC$ is a sufficient notion for algebraic derivation.

\section{Computing $\nabla \AMC$}\label{sec:compute_grad}

Conditioning a formula is straightforward, and hence computing $\nabla \AMC$ can be done with any off-the-shelf AMC solver. In other words, all the results of \citet{kimmig_algebraic_2017} that link circuit and semiring properties transfer directly to $\nabla \AMC$. Naively computing each element in $\nabla \AMC$ closely equals forward mode differentiation and was already proposed for WMC by \citet{sang_performing_2005} to compute the conditionals of Bayesian inference. 

Forward mode differentiation is well-known for scaling linearly in the number of input variables. Reverse mode differentiation, better known as \textit{backpropagation}, is a dynamic programming algorithm that scales linearly in the number of output variables, which is usually constant. 
The backpropagation algorithm can easily be extended to work over semirings \cite{du_generalizing_2023}, and can hence compute $\nabla \AMC$.

As a dynamic programming algorithm, the downside of backpropagation is its memory use. More precisely, the naive backpropagation algorithm on a circuit has linear memory complexity in the number of edges the circuit. \citet{shih_smoothing_2019} already demonstrated that when the semiring is a semifield, i.e. there is a division operation, this memory complexity reduces to linear in the number of nodes of the circuit. Given that the number of edges in a circuit can be up to the square of the number of nodes, this forms a substantial improvement. We show that this semifield requirement can be dropped while retaining the same memory complexity.

\begin{theorem}\label{thm:cancel}
    The backward pass on an algebraic circuit $C$ has $O(e)$ time and $O(n)$ memory complexity, where $e$ and $n$ are the number of edges and nodes in $C$ respectively.
\end{theorem}

Theorem~\ref{thm:cancel} is realised by Algorithm~\ref{alg:backprop}. This algorithm assumes that the forward pass already computed the values of sum nodes as $\alpha(n) = \bigoplus_{c \in \text{children}(n)} \alpha(c)$ and product nodes as  $\alpha(n) = \bigotimes_{c \in \text{children}(n)} \alpha(c)$. Algorithm~\ref{alg:backprop} then performs backpropagation in the usual way, going over the circuit from the root to the leaves and calculating the gradients of the children of each node using the chain rule.  
Concretely, the derivative towards a node $n$ is
\begin{equation}
\gamma(n) = \bigoplus_{p \in \text{parents}(n)} \gamma(p)
\end{equation}
when $n$ has sum nodes as parents (lines 5-7), or
\begin{equation}\label{eq:back_times}
\gamma(n) = \bigoplus_{p \in \text{parents}(n)} \bigotimes_{c \in \text{children}(p) \setminus \{n \}} \alpha(c)
\end{equation}
when $n$ has product nodes as parents (lines 9-19).

Algorithm~\ref{alg:backprop} relies on some dynamic programming to avoid recomputing the inner product in Equation~\ref{eq:back_times} for every parent. For example, taking the gradient of $c_1 \otimes c_2 \otimes c_3$ requires us to compute $[c_2 \otimes c_3, c_1 \otimes c_3, c_1 \otimes c_2]$. Naively, computing these individually will result in quadratic time complexity. We avoid this using two cumulative products: the forward $[e^\otimes, c_1, c_1\otimes c_2]$ and the backward $[c_3 \otimes c_2, c_2, e^\otimes]$. Both of these cumulative products can be computed in linear time, and their element-wise product results in the desired gradient.

Further optimizations on Algorithm~\ref{alg:backprop} are possible depending on the semiring structure to avoid the need to store the cumulative products in the backpropagation of the product nodes (lines 9-14). A first property that can achieve this is cancellation.

\begin{definition}
A semiring element $a \in \A$ is (multiplicatively) \emph{cancellative} if, for each $b$ and $c$ in $A$, $a \otimes b = a \otimes c$ implies $b = c$.
\end{definition}

Cancellation can be seen as a generalization of inverses. So when $c = a \otimes b$ and $a$ is cancellative we can write $b = c \oslash a$. Indeed, the inverse $c \oslash a$ must exist as $c = a \otimes b$, and the cancellation property furthermore assures that it is unique.
The semiring of the natural numbers $\textsc{Nat}$ form an example where we can use cancellation even though it is not a ring and there is no inverse in general.

If a product node $n$ is multiplicatively cancellative, we can simply compute the product gradient as
\[ \gamma(n) = \bigoplus_{p \in \text{parents}(n)} \alpha(p) \oslash \alpha(n) \]

A second property we can exploit is ordering.
\begin{definition}
We call $a,b\in \A$ (multiplicatively) \textit{ordered} when $a \otimes b =a$ or $a \otimes b = b$.
\end{definition}
Observe that $a \leq b \Leftrightarrow a \otimes b = a$ forms a partial order in semirings with idempotent multiplication.
So when the children of a node $n$ are all ordered,
the product derivatives are simply $\gamma[n] \otimes \alpha[n]$ for all children except the largest child. For the largest child, the derivative is $\gamma[n]$ times the one-but-largest child. 

The $\textsc{Fuzzy}$ semiring is an example of a semiring where all elements are multiplicatively ordered.
Note that barring the identity, ordering and cancellation are mutually exclusive, meaning they can be used complementarily. For example, in the $\textsc{Prob}$ and $\textsc{Nat}$ semirings, all elements are cancellative except zero which is ordered. So we can always exploit either cancellation or ordering, depending on the value of the node. We present a variation of Algorithm~\ref{alg:backprop} in the appendix which exploits both these cancellation and ordering properties.

\begin{algorithm}[tb]
\caption{Algebraic Backpropagation}
\label{alg:backprop}
\textbf{Input}: A circuit $C$ over the literals $\Ell$, a labelling of the nodes $\alpha$ (computed in the forward pass), and a semiring $(\A, \oplus, \otimes, e^\oplus, e^\otimes)$.\\
\textbf{Output}: The algebraic gradient $\nabla \AMC(C; \alpha)$.

\begin{algorithmic}[1] %
\STATE $\gamma \leftarrow [e^\oplus, \dots, e^\oplus]$, a vector with $e^\oplus$ for each node in $C$
\STATE $\gamma[\text{root of }C] \leftarrow e^\otimes$
\FOR{nodes $n$ in the circuit $C$, parents before children}
\IF {$n$ is a sum node}
\FOR{$c$ in children($n$)}
\STATE {$\gamma[c] \leftarrow \gamma[c] \oplus \gamma[n]$}
\ENDFOR
\ELSIF{$n$ is a product node}
\STATE {$r \leftarrow [e^\otimes,\dots,e^\otimes]$, with length $\lvert \text{children}(n)\rvert$}
\STATE {$t \leftarrow e^\otimes$}
\FOR{$c$ in children($n$)}
\STATE {$r[c] \leftarrow t$}
\STATE {$t \leftarrow t \otimes \alpha[c]$}
\ENDFOR
\STATE {$t \leftarrow e^\otimes$}
\FOR{$c$ in children($n$) in reverse order}
\STATE {$\gamma[c] \leftarrow \gamma[c] \ \oplus\ \gamma[n] \otimes t \otimes r[c]$}
\STATE {$t \leftarrow t \otimes \alpha[c]$}
\ENDFOR
\ENDIF
\ENDFOR
\STATE \textbf{return} $\gamma$
\end{algorithmic}
\end{algorithm}

\section{Semirings for $\nabla \AMC$}\label{sec:semirings}

We now explore the use of $\nabla \AMC$ for first-order optimization and related applications. In the probabilistic setting, we assume that our formula $\phi$ has weights $\alpha(x) \in [0, 1]$ with $\alpha(\neg x) = 1 - \alpha(x)$. In other words, the variables correspond to Bernoulli distributions, and we have a probability distribution over interpretations $p(I;\alpha) = \prod_{l \in I} \alpha(l)$. The weighted model count can then also be seen as a probability, which we denote $p(\phi; \alpha) = \AMC_\textsc{prob}(\phi; \alpha)$.

\paragraph{\textsc{Prob} semiring} By construction, the algebraic gradient $\nabla \AMC$ in the \textsc{Prob} semiring is just the usual gradient. 
\[\nabla \AMC_{\textsc{Prob}} (\phi; \alpha) = \nabla_{\!\alpha\,} p(\phi; \alpha)\]

\paragraph{\textsc{Log} semiring} Next, if we take the \textsc{Log} semiring, we instead get the logarithm of the gradient. Note that this is different from $\nabla_{\! \alpha\,} \log p(\phi; \alpha)$, the gradient of the log probability.
\[\nabla \AMC_{\textsc{Log}} (\phi; \alpha) = \log \nabla_{\! \alpha\,} p(\phi; \alpha)\] %

Backpropagation with the \textsc{Log} semiring provides a more numerically stable way to compute gradients. It can also be applied for Expectation-Maximization (EM), as the expectation step on logic formulas reduces to computing the conditionals $p(x \mid \phi)$ \cite{peharz_einsum_2020}.
\[ p(x \mid \phi) = \exp\left( \nabla \AMC_{\textsc{Log}}(\phi; \alpha) + \alpha - \AMC_{\textsc{Log}}(\phi; \alpha) \right)\]
The above also closely relates to the PC flow as defined by \citet{choi_group_2021}.

\paragraph{\textsc{Viterbi} semiring} In the $\textsc{Viterbi}$ semiring, the AMC gradient computes the maximum gradient over all models. Here, $\max$ computes the element-wise maximum of the gradient vectors of each model of $\phi$.
\[ \nabla \AMC_{\textsc{Viterbi}} (\phi; \alpha) = \max_{I \in \mathcal{M}(\phi)} \nabla_{\!\alpha\,} p(I; \alpha) \]
Use cases for this include greedily approximating the true gradient or gradient-based interpretability. Similarly, the \textbf{\textsc{Tropical} semiring} gives the log-space equivalent of the \textsc{Viterbi} semiring.
\[ \nabla \AMC_{\textsc{Tropical}} (\phi; \alpha) = \log \max_{I \in \mathcal{M}(\phi)} \nabla_{\!\alpha\,} p(I; \alpha) \]

\paragraph{\textsc{Grad} semiring} AMC with the \textsc{Grad} semiring computes the Shannon entropy \cite{li_first-_2009}. 
\[H(\phi) = -\sum_{I \in \mathcal{M}(\phi)} p(I; \alpha) \log p(I; \alpha)\] 
Taking the gradient of the AMC in \textsc{Grad} hence results in the conditional Shannon entropy towards each literal.
\[ \nabla \AMC_{\textsc{Grad}}(\phi) = \left[H(\phi \mid x_1), \dotsm H(\phi \mid x_n)\right]^\top \]

This conditional entropy is used as the information gain for learning, or for interpretability or regularization purposes.
Several other statistical quantities such as the KL divergence can be framed as the difference between the entropy and conditional entropy, and can hence easily be computed from the $\textsc{Grad}$ semiring.

\paragraph{\textsc{Bool} semiring with sampled labels} By sampling from the labels $\alpha$, we can combine the Boolean semiring \textsc{Bool} with a stochastic labelling. This implements a naive Monte Carlo approximation, which equals the true probability in expectation. 
\[ \E_{w\sim \alpha}[\AMC_\textsc{Bool}(\phi; w)] = p(\phi; \alpha) \]

Interestingly, we can immediately take the algebraic gradient here to get an unbiased gradient estimator. 
\[ \E_{w \sim \alpha}[ \nabla \AMC_{\textsc{Bool}}(\phi; w) ] = \nabla_{\!\alpha\,} p(\phi; \alpha) \]
Upon closer inspection, this equals IndeCateR \cite{de_smet_differentiable_2023}, a state-of-the-art unbiased estimator for gradients of discrete distributions that Roa-Blackwellizes REINFORCE. Moreover, by using backpropagation we estimate the gradient in linear time in the size of the formula, as opposed to the quadratic time required in the original formulation by \citet{de_smet_differentiable_2023}.

\paragraph{Other semirings} Some remaining semirings have less clear gradient-based interpretations, but might still be of note. The gradient in the \textbf{\textsc{Sens} semiring} computes sensitivity polynomials for the conditioned formulas. The gradient of the \textbf{\textsc{Fuzzy} semiring} computes the highest minimal membership of a literal in each conditioned formulas. 
Ordered binary decision diagrams (OBDD) are a canonical circuit representation supporting conjunction and disjunction in polytime \cite{bryant_binary_2018}, and can hence also be employed as a semiring.
Taking the gradient in the \textbf{\textsc{OBDD} semiring} creates a multi-rooted circuit for the gradient using bottom-compilation (i.e. an OBDD circuit which explicitly computes the forward and backward pass).

\section{Second-Order Derivations}\label{sec:second_order}

So far, we only considered first-order methods such as gradient descent. Second-order methods such as Newton's method take curvature into account by preconditioning the gradient using the inverse Hessian. For a function $f$ parameterized by $\theta$, Newton's method updates the parameters as
\[ \theta \leftarrow \theta + [\nabla^2 f(\theta)]^{-1} \nabla f(\theta) \]

Computing the full Hessian matrix has quadratic complexity, making the cost of second-order methods for machine learning often prohibitive. However, tractable circuits support a wide range of inference operations in polytime that are otherwise NP-hard \cite{vergari_compositional_2021}. The question hence poses itself whether tractable circuits could improve upon this quadratic complexity. %

Similar to the gradient, we first generalize the Hessian matrix to AMC as follows.
\[\nabla^2 \AMC(\phi) = \begin{bmatrix}
\AMC(\phi {\mid} x_1, x_1) & {\dots} & \AMC(\phi {\mid} x_1, x_n) \\
\vdots &  & \vdots \\
\AMC(\phi {\mid} x_n, x_1) & {\dots} & \AMC(\phi {\mid} x_n, x_n) \\
\end{bmatrix}\]

We ignore the labels $\alpha$ here to simplify notation. By construction, $\nabla^2\AMC$ is the conventional Hessian in the $\textsc{Prob}$ semiring.
Note that although the algebraic Hessian is well-defined in any semiring, applying Newton's method requires that the Hessian can be inverted.

By using $\nabla \AMC$ in the $\textsc{Grad}$ semiring, we get a straightforward way to calculate $\nabla^2 \AMC$ for the $\textsc{Prob}$ semiring. Indeed, the $\textsc{Grad}$ semiring calculates partial derivatives using dual numbers, so in algebraic backpropagation this gives a row of the Hessian.
\[\left[ \frac{\partial^2 \AMC_\textsc{Prob}(\phi)}{\partial \alpha(x) \partial \alpha(y_1)},  \dots, \frac{\partial^2 \AMC_\textsc{Prob}(\phi)}{\partial \alpha(x) \partial \alpha(y_n)} \right] \]

Next, we prove that $\nabla^2 \AMC$ on tractable circuits still lacks structure, meaning that the expensive quadratic memory cost cannot avoided. 

\begin{theorem}\label{thm:hessian_quadratic}
    Representing {\normalfont $\nabla^2 \AMC(C)$} of a circuit $C$ over $v$ variables has a $O(v^2)$ memory complexity, even when $C$ is smooth, decomposable, and deterministic.
\end{theorem}

\begin{proofsketch}
    Take for example the binary Galois field as the semiring. Given a sequence of $v^2$ bits, we can now construct a formula $\phi$ such that the (flattened) Hessian equals this bit sequence. This means that, in general, the Hessian has no structure and a sub-quadratic memory complexity would imply lossless compression.
\end{proofsketch}

One might give two counter-arguments to the above theorem. First, gradient descent takes linear memory in the circuit size, but this circuit size might be up to exponential in the number of variables. Indeed, Theorem~\ref{thm:hessian_quadratic} does not rule out that the Hessian can be stored in linear memory complexity in the size of the circuit. Second, calculating the Hessian is typically not the true goal. The above result does not rule out the tractability of a matrix-free approach, where the preconditioned gradient $[\nabla^{2}\AMC(\phi)]^{-1} \nabla \AMC(\phi)$ is computed directly without constructing the Hessian explicitly. 
However, even when considering these arguments, the existence of an algorithm for Newton's method which runs in linear time complexity in the circuit size remains unlikely.

\begin{theorem}\label{thm:hessian_linear_impossible}
Given a circuit $C$ with $n$ nodes, there cannot exist a circuit $C'$ with size $O(n)$ that computes the preconditioned gradient $[\nabla^{2}\AMC(C)]^{-1} \nabla \AMC(C)$, even when $C$ is deterministic, decomposable and smooth.%
\end{theorem}
\begin{proof}
    In appendix.
\end{proof}

We further discuss in the appendix that relaxing the computational model of arithmetic circuits does not improve the situation much. More precisely, if there would exist any kind of algorithm that computes the preconditioned gradient in linear time in the size of the circuit, a quadratic matrix multiplication algorithm would be implied.
The above results indicate that, much like for deep learning, first-order optimization might still be preferable. Our theorems do not rule out that there might exist specific semirings which are not fields where the situation is better. However, the practical relevance of Hessians in these semirings is more limited.

\begin{table*}[t]
    \centering
    \begin{tabular}{l r r r r}
        \toprule
        Time (ms)  & \textsc{Prob} & \textsc{Bool} & \textsc{Log} & \textsc{Fuzzy}
        \\\midrule
        Kompyle & \textbf{\phantom{000}7.8 ± 0.1} & \textbf{\phantom{0}4.1 ± 0.0} & \textbf{\phantom{00}13.9 ± 0.4} & \textbf{\phantom{00}9.8 ± 0.3} \\
        \  - \textit{dynamic} (Algorithm~\ref{alg:backprop}) & \phantom{00}19.8 ± 0.2 & 14.6 ± 0.1 & \phantom{00}24.4 ± 0.2 & \phantom{0}16.9 ± 0.1 \\
        \ - \textit{naive + cancel.} \cite{shih_smoothing_2019} & \phantom{0}203.1 ± 0.8 & 10.6 ± 0.0 & \phantom{0}167.0 ± 0.5 & / \\
        \ - \textit{naive} \cite{du_generalizing_2023} & \phantom{0}289.0 ± 1.3 & 10.9 ± 0.0 & \phantom{0}238.0 ± 1.2 & 174.7 ± 0.7 \\\midrule
        Pytorch & 7662.0 ± 155.1 & / & 7620.6 ± 136.7 & / \\
        Jax & Out of Memory & / & Out of Memory & / \\\bottomrule
    \end{tabular}
    \caption{Runtime in milliseconds to compute the algebraic gradient $\nabla \AMC$ averaged over 100 instances from the MC2021 competition (track 2). We report the average and standard deviation over 10 runs. 
    See Section~\ref{sec:experiments} for an explanation of the ablations of Kompyle. Methods that did not run on all circuits within 16GB of memory are denoted ``Out of Memory". PyTorch and Jax cannot compute $\nabla \AMC$ in the \textsc{Bool} and \textsc{Fuzzy} semirings.}
    \label{tab:experiments}
\end{table*}

\section{Related Work}\label{sec:related_work}

\paragraph{Semirings \& Inference} The semiring perspective in artificial intelligence has its roots in weighted automata \cite{schutzenberger_definition_1961}, as the weights of an automata can be defined over a semiring. These ideas have been carried over to various other paradigms such as constraint satisfaction problems \cite{bistarelli_semiring-based_1997}, parsing \cite{goodman_semiring_1999}, dynamic programs \cite{eisner_compiling_2005}, database provenance \cite{green_provenance_2007}, logic programming \cite{kimmig_algebraic_2011}, propositional logic \cite{kimmig_algebraic_2017}, answer set programming \cite{eiter_weighted_2020}, Turing machines \cite{eiter_semiring_2023}, and tensor networks \cite{goral_model_2024}.

\paragraph{Semirings \& Learning} While semirings have mostly been applied to inference problems, some works also investigated learning in semirings. \citet{li_first-_2009} introduced the expectation semiring, which can compute gradients and perform expectation-maximization. \citet{pavan_constraint_2023} studied the complexity of constraint optimization in some semirings. On the other hand, we provide a more general framework to compute derivations of any semiring.

\citet{darwiche_tractable_2001} already described a forward-backward algorithm for computing conditionals of a formula. Backpropagation on semirings has been described recently by \cite{du_generalizing_2023}. Most similar to our work, \citet{shih_smoothing_2019} already included an algorithm for computing the conditionals on algebraic circuits, which they call All-Marginals instead of $\nabla \AMC$. This algorithm can be seen as a special case of our cancellation optimization, as all cancellative commutative monoids can be embedded in a group using the Grothendieck construction. %

\paragraph{Neurosymbolic Learning} Several neurosymbolic methods rely on (probabilistic) circuits and hence could apply the algebraic learning framework we outlined. Some examples include DeepProbLog \cite{manhaeve_deepproblog_2018}, the semantic loss \cite{xu_semantic_2018}, and probabilistic semantic layers \cite{ahmed_semantic_2022}. \citet{dickens_modeling_2024} proposed another overarching view of neurosymbolic learning by framing it as energy-based models. On the other hand, our work focuses on algebraic circuits where inference and learning can be performed exactly.

\section{Experiments}\label{sec:experiments}

We implemented the algebraic backpropagation in a Rust library called \textit{Kompyle}, and empirically demonstrate the runtime performance of the algebraic backpropagation algorithm on several semirings. 

\paragraph{Setup} As a benchmark, we take 100 formulas of the 2021 Model Counting Competition \cite{fichte_model_2021} and compile them to d-DNNF circuits using the d4 knowledge compiler \cite{lagniez_improved_2017}. We randomly generate weights for the formulas, with on average 1\%

\paragraph{Results} Table~\ref{tab:experiments} contains the results. PyTorch and Jax perform poorly, as these frameworks are not optimized for very large computational graphs. Jax does not run within the 16GB of RAM, even on comparatively small circuits (ca. 100MB). Even though real-world circuits are far from the worst-case quadratic complexity, the dynamic programming considerably outperforms the naive variants of Kompyle.  The semiring optimizations yield smaller but still considerable speed-ups, depending on the semiring.

\section{Conclusion}

We proposed a notion of gradients for algebraic model counting as conditional inference. We showed that many quantities of interest in learning, such as gradients, Hessians, conditional entropy, etc. can be seen as this algebraic gradient in different semirings. Furthermore, we introduced an optimized backpropagation algorithm for a broad class of semirings. Finally, we gave an indication that second-order optimization is still expensive on tractable circuits.

\appendix

\section*{Appendix}
\subsection*{Gradient Semiring}

For completeness, the addition and multiplication operations in the \textsc{Grad} semiring are:
\begin{equation}\label{eq:grad_plus}
(a,b) \oplus (c,d) = (a+c,b+d)
\end{equation}
\begin{equation}\label{eq:grad_times}
(a,b) \otimes (c,d) = (a\cdot c, a\cdot d + c\cdot b)
\end{equation}

\subsection*{Proofs for Section~\ref{sec:amc_grad}}

We denote the derivation which maps all inputs to $e^\oplus$ as $\delta_0$. We furthermore lift addition and multiplication to derivations as follows. The addition of two derivations
$\delta_1 \oplus \delta_2$ is a derivation $\delta_3$ such that $ \delta_1(a) \oplus \delta_2(a) = \delta_3(a)$ for all $a$ in $\A$. Similarly, the scalar multiplication of an element $a$ in $\A$ with a derivation $\delta$ is a new derivation $\delta'$ such that $\forall b \in \A: \delta'(b) = a \otimes \delta(b)$.

\begin{theorem}
The derivations over the commutative semiring $(\A, \oplus, \otimes, e^\oplus, e^\otimes)$ are a commutative monoid $(\mathcal{D}(A), \oplus, \delta_0)$. 
\end{theorem}

\begin{proof}
We verify the monoid axioms. First, we can see that the set of derivations is closed under addition as follows.
\begin{align*}
(\delta_1 \oplus \delta_2) (a \oplus b) &= \delta_1(a \oplus b) \oplus \delta_2(a \oplus b)\\
&= \delta_1(a) \oplus \delta_2(a) \oplus \delta_1(b) \oplus \delta_2(b) \\
&= (\delta_1 \oplus \delta_2)(a) \oplus (\delta_1 \oplus \delta_2)(b)
\end{align*}
\begin{align*}
(\delta_1 \oplus \delta_2) (a \otimes b) &= \delta_1(a \otimes b) \oplus \delta_2(a \otimes b)\\
&= b {\otimes} (\delta_1(a) \oplus \delta_2(a)) \oplus a {\otimes}(\delta_1(b) \oplus \delta_2(b)) \\
&= (b {\otimes} (\delta_1 \oplus \delta_2))(a) \oplus (a {\otimes} (\delta_1 \oplus \delta_2))(b)
\end{align*}
Second, the neutrality of $\delta_0$ holds as
\[ (\delta \oplus \delta_0)(a) = \delta(a) \oplus \delta_0(a) = \delta(a) \oplus e^\oplus = \delta(a) \]
Third, the commutativity of adding derivations follows directly from the commutativity of the semiring addition.
\[ (\delta_1 \oplus \delta_2)(a) = \delta_1(a) \oplus \delta_2(a) = \delta_2(a) \oplus \delta_1(a) = (\delta_2 \oplus \delta_1)(a) \]
Associativity can be shown analogously.
\end{proof}

\begin{theorem}
The derivations over the commutative semiring $(\A, \oplus, \otimes, e^\oplus, e^\otimes)$ are an $\A$-semimodule.
\end{theorem}
\begin{proof}
By verification of the semimodule axioms.
\[ \left(a {\otimes} (\delta_1 \oplus \delta_2)\right)(b) = a \otimes ((\delta_1 \oplus \delta_2) (b)) = a \otimes (\delta_1(b) \oplus \delta_2(b)) \]
\[((a \oplus b) \otimes \delta)(c) = (a \oplus b) \otimes \delta(c)  = (a \otimes \delta \oplus b \otimes \delta)(c) \]
\[ ((a \otimes b) \otimes \delta)(c) = (a \otimes (b \otimes \delta))(c) \]
\[ (e^\otimes \otimes \delta)(a) = e^\otimes \otimes \delta(a) = \delta(a) \]
\[ (e^\oplus \otimes \delta)(a) = e^\oplus = \delta_0(a) \qedhere\]
\end{proof}

\begin{theorem}
{\normalfont $\nabla \AMC$} is a basis for the $\F_\V$-semimodule over semiring derivations.
\end{theorem}
\begin{proof}
We write $\delta_l$ for the derivation towards the literal $l$, meaning that $\delta_l(l) = e^\otimes$ and $\delta_l(l') = e^\oplus$ for all other literals $l' \neq l$. $\nabla \AMC$ consists precisely of the literal derivations $\delta_l(\phi) = \phi {\mid} l$. To show that it is a basis, we show that the $\delta_l$ are indepedent can generate any derivation.

To see the linear indepedence, consider that $a \otimes \delta_l'(l) = e^{\oplus}$ for all $a \in A$ and $l' \neq l$. In other words, any linear combination of the other literal derivations outputs zero on $l$. To see that the basis spans all derivations, consider that we can decompose any derivation $\delta$ as follows.
\[ \delta = \bigoplus_{l \in \Ell} \delta(l) \otimes \delta_{l}  \]

To see that this decomposition is correct, observe that for any literal $l' \in \Ell$:
\[  \left(\bigoplus_{l \in \Ell} \delta(l) \otimes \delta_{l}\right)(l') = \delta(l') \otimes \delta_{l'}(l') = \delta(l') \]
So the decomposition behaves the same on literals, and by an induction argument on all other elements in $\F_\V$.
\end{proof}

\subsection*{Proofs for Section~\ref{sec:second_order}}\label{app:hessian_free}

We first introduce an auxiliary construction to convert from matrices to formulas. By abuse of notation, the following theorems will refer to the algebraic gradient and Hessian of a formula/circuit as those only to the positive literals. As the semiring, we use the binary Galois field $\mathbb{F}_2$ over the Booleans, where addition is XOR and multiplication is AND.

\begin{theorem}\label{thm:hessian_matrix_reduction}
Given a binary $n\times n$ matrix $M$, a circuit $C$ over $n$ variables can be created in quadratic time such that
\begin{itemize}
    \item the Hessian of the positive literals of $C$ equals $M$ in $\mathbb{F}_2$.
    \item the size of $C$ is $\Theta(n^2)$.
    \item the circuit $C$ is smooth, deterministic and decomposable.
\end{itemize}
\end{theorem}
\begin{proof}

First, note that each entry in the algebraic Hessian has a unique model, meaning that no other entry has this model. Namely, for $\AMC(C {\mid} x_i, x_j)$ this is the model where $x_i$ and $x_j$ are positive and all other literals negative. This means that given a binary matrix $M$, we can easily construct a circuit $C$ which has $\nabla^2 \AMC_{\mathbb{F}_2}(C) =M$ in the following fashion. For each element $i,j$ which is 1 in $M$ add a cube for the unique model.
This procedure creates a DNF formula that is smooth, deterministic, and decomposable; also known as the MODS representation.

The above naive encoding creates a circuit of $O(n^3)$, as there can be up to $n^2$ cubes in the DNF and each cube has length $n$. However, using a dynamic programming approach we can reduce the circuit size to $O(n^2)$. First, we precompute a circuit for the elements of the following matrix, where $l_i$ are all the negative literals.
\[\begin{bmatrix}
    l_1 & l_1 \land l_2 & \dots & \bigwedge_{1}^{n-2} l_i & & \\
    & l_2 & l_2 \land l_3 & \dots & \bigwedge_{2}^{n-1} l_i & \\
    & & l_3 & l_3 \land l_4 & \dots & \bigwedge_{2}^{n} l_i \\
     & & & & \ddots & \\
    & & & & l_{n-1} & l_{n-1} \land l_n \\
    & & & & & l_n
\end{bmatrix}\]
Each row can be represented in linear time as a cumulative conjunction, and hence there is a circuit for this matrix of size $\Theta(n^2)$.
Now we represent each cube in the DNF as $x_i \land x_j \land r$, where the remainder $r$ for all the negative literals is a conjunction of at most 3 elements of the precomputed matrix above.
\end{proof}

\begin{theorem}\label{thm:hessian_grad_reduction}
Given a binary $n\times n$ matrix $M$ and a vector $v$ with $n$ elements, a circuit $C$ over $n$ variables can be created in quadratic time such that
\begin{itemize}
    \item $\nabla^2 \AMC_{\mathbb{F}_2}(C) = M$.
    \item $\nabla \AMC_{\mathbb{F}_2} (C) = v$.
    \item the size of $C$ is $\Theta(n^2)$.
    \item the circuit $C$ is smooth, deterministic and decomposable.
\end{itemize}
\end{theorem}
\begin{proof}
    We start with the same circuit construction as in Theorem~\ref{thm:hessian_matrix_reduction}. Now, we only need to additionally guarantee that this circuit has the correct algebraic gradient. For each $\AMC(C{\mid}x_i)$, we again have a unique model where $x_i$ is true and all literals are false. We first marginalize over the hessian entries for this gradient entry $\bigoplus_j \AMC(C{\mid}x_i x_j) $, and then set the model such that the XOR results in $v_i$. Clearly, this procedure adds at most $n$ models to our DNF, and hence remains quadratic.
\end{proof}

\begin{theorem}\label{thm:inv_hessian_hard}
If there exists an algorithm that computes the inverse of the Hessian $\nabla^2 \AMC(C)$ of a circuit $C$ in $O(\lvert C \rvert)$, matrix inversion over $\mathbb{F}_2$ is possible in $O(n^2)$.
\end{theorem}
\begin{proof}
We construct a quadratic time matrix inversion algorithm as follows. 
Given a matrix $M$, we can create a circuit $C$ which as $M$ as its Hessian as described in Theorem~\ref{thm:hessian_matrix_reduction}. Now if we apply the assumed algorithm, we get the inverse Hessian which is just $M^{-1}$. As the assumed algorithm works in linear time in the circuit, and the circuit is quadratic in $n$, this full method has $O(n^2)$ time complexity.
\end{proof}

\begin{theorem}\label{thm:hessian_free_hard}
If there exists an algorithm that computes the preconditioned gradient $\nabla^2 \AMC(C)^{-1} \nabla \AMC(C)$ of a circuit $C$ in $O(\lvert C \rvert)$, linear systems over $\mathbb{F}_2$ can be solved in $O(n^2)$.
\end{theorem}
\begin{proof}
    Idententical to the proof of Theorem~\ref{thm:inv_hessian_hard}, but now using the construction of Theorem~\ref{thm:hessian_grad_reduction}.
\end{proof}

The impossibility statement in Theorem~\ref{thm:hessian_linear_impossible} follows due to the $\Omega(n^2 \log n)$ lower bound on the circuit complexity of matrix multiplication by \citet{raz_complexity_2002}.

\subsection*{Optimizations of Algebraic Backpropagation}\label{app:backprop_opt}

We reconsider Algorithm~\ref{alg:backprop} with the optimizations discussed in Section~\ref{sec:compute_grad}. As only the backpropagation of products is affected, we just restate this part of the algorithm.

\begin{algorithm}
\caption{Algebraic backpropagation through a product, exploiting cancellation and ordering.}
\label{alg:backprop_opt}
\textbf{Input:} product node $n$.
\begin{algorithmic}[1] %
\FOR{child $c$ in children($n$)}
\IF{$\alpha(c)$ is cancellative}
\STATE{$\gamma(c) \leftarrow \gamma(c) \oplus \alpha(n) \oslash \alpha(c)$}
\ELSIF{$\alpha(n)$ is the maximal element of children($n$) and $\alpha(n) \neq \alpha(c)$ or $\alpha(c)$ occurs twice}
\STATE{$\gamma(c) \leftarrow \gamma(c) \oplus \alpha(n)$}
\ELSE
\STATE{// Naive fall-back.}
\STATE{$\gamma(c) \leftarrow \gamma(c) \oplus \bigotimes_{c' \in \text{children}(n) \setminus \{c\}}\alpha(c')$}
\ENDIF
\ENDFOR
\end{algorithmic}
\end{algorithm}

\section*{Acknowledgments}
This research received funding from the Flemish Government (AI Research Program), the Flanders Research Foundation (FWO) under project G097720N, KUL Research Fund iBOF/21/075, and the European Research Council (ERC) under the European Union’s Horizon 2020 research and innovation programme (Grant agreement No. 101142702). Luc De Raedt is also supported by the Wallenberg AI, Autonomous Systems and Software Program (WASP) funded by the Knut and Alice Wallenberg Foundation.

JM wants to thank Pedro Zuidberg Dos Martires for stimulating discussions, as well as Vincent Derkinderen, David Debot, and Sieben Blocklandt for proofreading a draft manuscript.

\bibliography{references}

\end{document}